\documentclass{article} % For LaTeX2e
\usepackage{nips14submit_e,times}
\usepackage{hyperref}
\usepackage{url}
\usepackage{algorithm}
\usepackage{algorithmic}
\usepackage{amsthm}
\usepackage{amsmath}
\usepackage{amssymb}
\usepackage{bm}
\usepackage{dsfont}
\usepackage[numbers]{natbib}
\usepackage{graphicx}
\usepackage{caption}
\usepackage{subcaption}
\usepackage{xspace}
\usepackage{wrapfig}
\usepackage{bbm}

\newtheorem{lemma}{Lemma}
\newtheorem{theorem}{Theorem}

\newtheorem{corollary}{Corollary}

\newcommand{\I}{{\mathds{1}}}

\newcommand{\etat}{\eta_t}
\newcommand{\gammat}{\gamma_t}
\newcommand{\nodes}{d}
\newcommand{\regret}{R_T}
\newcommand{\transpose}{^\mathsf{\scriptscriptstyle T}}

\newcommand{\sumt}{\sum_{t=1}^T}

\newcommand{\sumj}{\sum_{j\in N_i^-}}
\newcommand{\sumtj}{\sum_{j\in N_{t,i}^-}}
\newcommand{\sumi}{\sum_{i=1}^{\nodes}}

\newcommand{\sumtji}{\sum_{j\in \{N_{t,i}^-\cup\{i\}\}}}

\newcommand{\hdk}{\hat{d}_k^-}

\newcommand{\ti}{_{t,i}}

\newcommand{\pti}{p\ti}

\newcommand{\qti}{q\ti}

\newcommand{\hpi}{\hat{p}_i}
\newcommand{\hp}{\hat{p}}

\newcommand{\ptj}{p_{t,j}}
\newcommand{\qtj}{q_{t,j}}

\newcommand{\oti}{o\ti}

\newcommand{\loss}{\ell}
\newcommand{\hloss}{\hat{\ell}}
\newcommand{\hLoss}{\hat{L}}

\newcommand{\real}{\mathbb{R}}
\newcommand{\Sw}{\mathcal{S}}

\newcommand{\II}[1]{\mathds{1}_{\left\{#1\right\}}}

\newcommand{\EE}[1]{\mathbb{E}\left[#1\right]}

\newcommand{\PPc}[2]{\mathbb{P}\left[#1\left|#2\right.\right]}

\newcommand{\OO}{\mathcal{O}}
\newcommand{\EEc}[2]{\mathbb{E}\left[#1\left|#2\right.\right]}
\newcommand{\EEcc}[2]{\mathbb{E}\left[\left.#1\right|#2\right]}

\newcommand{\ev}[1]{\left\{#1\right\}}
\newcommand{\pa}[1]{\left(#1\right)}
\newcommand{\F}{\mathcal{F}}

\DeclareMathOperator*{\argmin}{\arg\min}
\newcommand{\hL}{\wh{L}}

\newcommand{\norm}[1]{\left\|#1\right\|}
\newcommand{\onenorm}[1]{\norm{#1}_1}

\newcommand{\bV}{\boldsymbol{V}}
\newcommand{\bO}{\boldsymbol{O}}

\newcommand{\bv}{\boldsymbol{v}}
\newcommand{\tbV}{\wt{\bV}}
\newcommand{\tbZ}{\wt{\bZ}}
\newcommand{\tQ}{\wt{Q}}

\newcommand{\bloss}{\bm\ell}
\newcommand{\bL}{\boldsymbol{L}}

\newcommand{\tV}{\widetilde{\bV}}
\newcommand{\hbl}{\hat{\bloss}}

\newcommand{\hbL}{\wh{\bL}}

\newcommand{\talpha}{\widetilde{\alpha}}
\newcommand{\bmu}{\bm{\mu}}
\newcommand{\bp}{\boldsymbol{p}}
\newcommand{\bZ}{\boldsymbol{Z}}

\newcommand{\R}{\mathbb{R}}

\newcommand{\wh}{\widehat}
\newcommand{\wt}{\widetilde}

\newcommand{\fpl}{{FPL}\xspace}

\newcommand{\exph}{{\sc Exp3}\xspace}

\newcommand{\ra}{\rightarrow}

\newcommand{\expfive}{{\sc Exp3-IX}\xspace}
\newcommand{\expset}{{\sc Exp3-SET}\xspace}
\newcommand{\expdom}{{\sc Exp3-DOM}\xspace}

\newcommand{\fplix}{{\sc FPL-IX}\xspace}
\newcommand{\comphedge}{{\sc Component\-Hedge}\xspace}

\title{Efficient learning by implicit exploration in bandit problems with side observations}

\author{
Tom\'a\v s Koc\'ak \qquad Gergely Neu \qquad Michal Valko \qquad R\'emi 
Munos\thanks{Current affiliation: Google DeepMind}\\
SequeL team, INRIA Lille -- Nord Europe, France\\
\texttt{\small \{tomas.kocak,gergely.neu,michal.valko,remi.munos\}@inria.fr} \\
}

 \nipsfinalcopy % Uncomment for camera-ready version

\begin{document}

\maketitle

\begin{abstract}
We consider online learning problems under a partial observability model capturing situations where the
information conveyed to the
learner is between full information and bandit feedback. In the simplest
variant, we assume that in
addition to its own loss, the learner also gets to observe losses of some other actions. The revealed losses depend
on the learner's action and a directed  \textit{observation system} chosen by
the environment. For this setting, we propose
the first algorithm that enjoys near-optimal regret guarantees without having to know the observation system before
selecting its actions. Along similar lines, we also define a new partial
information setting that models online
combinatorial optimization problems where the feedback received by the learner is between semi-bandit and full
feedback. As the predictions of our first algorithm cannot be always computed efficiently in this setting, we
propose another algorithm with similar properties and with the benefit of always being computationally efficient, at the
price of a slightly more complicated tuning mechanism. Both
algorithms rely on a novel exploration strategy called \emph{implicit exploration}, which is shown to be more efficient
both computationally and information-theoretically than previously studied exploration strategies for the problem.

% We consider the online combinatorial optimization setting  used by \citet{CL12,audibert13regret}, with additional side
% information encoded in observability graph introduced by \citet{SM11}, and later used by \citet{alon2013from}, which can
% change over time. In this setting the learner has access to a possibly huge action set $S\subset\{0,1\}^\nodes$
% consisting of combinatorial substructures of at most $m$ components out of $\nodes$ possible. Each action is represented
% by a binary vector of dimension $\nodes$ and $L^1$-norm at most $m$. By playing an action we receive the losses of each
% individual components of an action and observe losses of some other components not included in an action played,
% specified by an observability graph. We have designed two algorithm dealing with this setting. The first algorithm is
% for the special case, where $m = 1$, therefore only individual components are played, and is based on Exp3. The second
% algorithm is based on FPL. Moreover, both algorithms do not need to access the 
% observability graph before an action is picked, and therefore computationally less expensive. Our main result is a characterization of the regret bounds in terms of independence numbers of the observability graph.
\end{abstract}

%%%%%%%%%%%%%%%%%%%%%%%%%
%%%%%%%%%%%%%%%%%%%%%%%%%
 \vspace{-0.5em}
\section{Introduction}
\vspace{-0.5em}
%%%%%%%%%%%%%%%%%%%%%%%%%
%%%%%%%%%%%%%%%%%%%%%%%%%

% The problem setting, which we are considering naturally arises in many real world problems. Let us consider two {t,i}s
% \begin{example}
% Suppose the recommendation problem, where we want to offer a product to a small group of social network users and then
% obtain a feedback, for example whether the user likes the product, or not. Moreover, all the friends (or followers) of
% the addressed users are able to see the offer and they will respond with some probability. We can model this problem as
% the (possibly oriented) graph with weighted edges. Our task is to choose a group of users, which is represented by a
% group of nodes of a graph, and then receive the feedback from each individual user in the group. Moreover each friend of
% the selected users can also provide the feedback with some probability, which is in our model represented by the
% neighbors of the selected graph, and the probabilities of providing the feedback are the weights of the edges. Our task
% is to choose in each time step the group of the users to maximize our profit.
% \end{example}

Consider the problem of sequentially recommending content for a set of users. 
In each period of this online decision
problem, we have to assign content from a news feed
% (offered by a website like cnn.com or bbc.com)
to each of our subscribers so as to maximize clickthrough. We assume that this
assignment needs to be done well in advance, so that we
only observe the actual content after the assignment was made and the user had
the opportunity to click. While
we can easily formalize the above problem in the classical multi-armed bandit framework \citep{auer2002bandit}, notice
that we will be throwing out important information if we do so! The additional information in this problem comes from
the fact that several news feeds can refer to the same content, giving us the 
opportunity to infer clickthroughs for
a number of assignments that we \emph{did not actually make}. For example, consider the situation shown on
Figure~1a. 
In this simple example, we want to suggest one out of three news feeds to each user, that is, we want to choose a
matching on the graph shown on Figure~1a which covers the users. Assume that
news feeds~2 and~3 refer
to the same content, so \emph{whenever we assign news feed~2 or~3 to any of 
the users, we learn the value of both of
these assignments}. The relations between these assignments can be described by a graph structure
(shown on Figure~1b), where nodes represent user-news feed assignments, and
edges mean that the
corresponding assignments reveal the clickthroughs of each other. For a more compact representation, we can group the
nodes by the users, and rephrase our task as having to choose one node from each group. Besides its own
reward, each selected node reveals the rewards assigned to all their neighbors.

\begin{figure}[H]
\centering
\begin{subfigure}[t]{0.60\textwidth}
\centering
\includegraphics[height = 3cm]{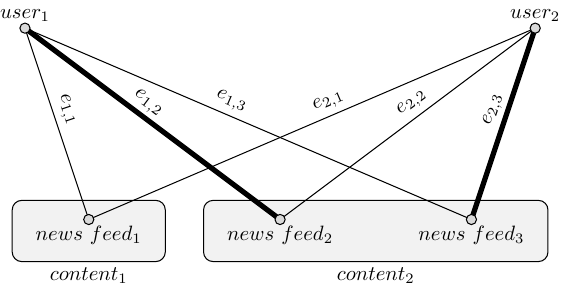}
\caption*{Figure~1a: Users and news feeds. The thick edges represent one
potential matching of users to feeds, grouped news feeds show the same content.}
\label{fig:newsFeedsA}
\end{subfigure}
\quad
\begin{subfigure}[t]{0.295\textwidth}
\centering
\includegraphics[height = 3cm]{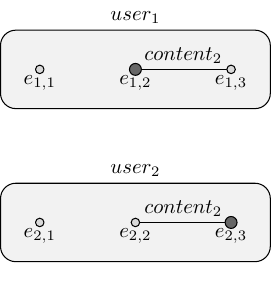}
\caption*{Figure~1b: Users and news feeds. Connected feeds mutually reveal each
others clickthroughs.}
\label{fig:newsFeedsB}
\end{subfigure}
\end{figure}

The problem described above fits into the framework of \emph{online combinatorial optimization} where in each round, a
learner selects one of a very large number of available actions so as to minimize the losses associated with its
sequence of decisions. Various instances of this problem have been widely studied in recent years under different
feedback assumptions \cite{CL12,audibert2014regret,CWY13}, notably including 
the so-called \emph{full-information}
\cite{koolen10comphedge} and \emph{semi-bandit}
\cite{audibert2014regret,neu2013efficient} settings. Using the example in 
Figure~1a,
assuming full information means that clickthroughs are observable for \emph{all}
assignments, whereas assuming semi-bandit feedback, clickthroughs are only observable on the actually realized
assignments. While it is unrealistic to assume full feedback in this setting, assuming semi-bandit feedback is far too
restrictive in our example. Similar situations arise in other practical
problems such as packet routing in computer networks where we may have
additional information on the delays in the network besides the delays of our
own packets.
% or \todoG{or...?}.

In this paper, we generalize the partial observability model first proposed 
by~\citet{mannor2011from}
and later revisited by \citet{alon2013from} to accommodate the feedback
settings situated between the full-information and the semi-bandit schemes.
Formally, we consider a sequential decision making problem  where in each time
step $t$ the (potentially \emph{adversarial}) \emph{environment}
 assigns a loss value to each out of $\nodes$ components, and generates
an \emph{observation system} whose role will be clarified soon. Obliviously of the environment's choices, the
\emph{learner} chooses an action $\bV_t$ from a fixed action set $S\subset\{0,1\}^\nodes$ represented by a binary
vector with at most $m$ nonzero components, and incurs the sum of losses associated with the nonzero components
of $\bV_t$. At the end of the round, the learner observes the individual losses along the chosen components
\emph{and some additional feedback based on its action and the observation
system}. We represent this 
observation system by a directed \emph{observability graph} with $\nodes$ 
nodes,
with an edge connecting $i\ra j$ if and only if the loss associated with $j$ is
revealed to
the learner whenever $V_{t,i}=1$.
The goal of the learner is to minimize its total loss obtained over $T$ repetitions of the above procedure. 
The two most well-studied variants of this general framework are the multi-armed
bandit problem~\cite{auer2002bandit} where each action consists of a single
component and the observability graph is a graph without
edges, and the problem of prediction with expert
advice~\cite{Vov90,LW94,cfhhsw:expert97} where each action consists of
exactly one component and the observability graph is complete. In the true
combinatorial setting where $m>1$, the empty and complete graphs correspond to
the semi-bandit and full-information settings respectively.

% We are interested in an intermediate case between a bandit or semi-bandit problem and an expert problem using
% observability graph. Similar problem for non-combinatorial setting i.e. actions consisting of exactly one nonzero
% component, is recently studied in \cite{SM11,alon2013from}, where after playing an action, a loss of the action is
% received, and losses of some specific subset of actions are observed. This subset of actions can be encoded in
% observability graph whose edges tells us, which losses are revealed after playing an
% arm. 

Our model directly extends the model of~\citet{alon2013from}, whose setup
coincides with~$m=1$ in our framework. \citeauthor{alon2013from}
themselves were motivated by the work of~\citet{mannor2011from}, who considered
\emph{undirected} observability systems where actions mutually uncover each 
other's losses. \citeauthor{mannor2011from} proposed
an algorithm based on linear programming that achieves a regret of 
$\tilde\OO(\sqrt{cT})$, where
$c$ is the number of cliques into which the graph can be split. While this bound
recovers the usual bounds of
$\tilde\OO(\sqrt{T})$ and $\tilde\OO(\sqrt{d T})$ in the experts and
bandit settings respectively, it fails to exploit the
opportunities offered by observation systems with large number of connections
but only small cliques. While in such
situations we would reasonably expect a regret guarantee closer to
$\tilde\OO(\sqrt{T})$ than $\tilde\OO(\sqrt{d T})$, the
bound of \citet{mannor2011from} only delivers the latter with  $c = \Omega(d)$.
% Suppose we have a problem with $T$ rounds and $\nodes$ possible actions, the regret bound of the bandit setting is of
% order $\sqrt{T\nodes}$, and regret bound for the expert setting is of order $\sqrt{T\ln \nodes}$. 
% The algorithm proposed
% in \cite{SM11} has regret bound of order $\sqrt{cT\ln \nodes}$, where $c$ is the number of cliques into which the graph
% can be split. For the case of graph with no edges, corresponding to the bandit setting, the algorithm gives us regret
% bound of order $\sqrt{\nodes T\ln \nodes}$ which is essentially the regret bound which we expect for the bandit setting.
% For the case of graph with edge connecting each pair of nodes, corresponding to the expert setting, the algorithm gives
% us regret bound of order $\sqrt{T\ln \nodes}$, which is also what we expect for expert setting. Even if this bound is
% interesting and gives us interesting result for the intermediate case, there are some graphs which have large number of
% edges, but only small cliques and therefore a large $c$. In this case we would like to get bound close to the one of
% expert setting, since we observe losses of many arms. On the other hand, we get regret bound close
% to the one of bandit setting because of large value of $c$.
This issue was addressed by \citet{alon2013from} who proposed an algorithm called \expset that guarantees a regret of
$O(\sqrt{\alpha T\log d})$, where $\alpha$ is an upper bound on the 
\emph{independence numbers} of the observability graphs assigned by the 
environment. In particular, \expset is much more efficient than the algorithm 
of \citeauthor{mannor2011from} as it
only requires running the \exph algorithm of \citet{auer2002bandit} on the decision set, which runs in time linear in
$d$. \citet{alon2013from} also extend the model of \citeauthor{mannor2011from} in  
allowing
the observability graph to be
directed. For this setting, they offer another algorithm called \expdom with similar guarantees, although with the
serious drawback that it \textit{requires access to the observation system
before choosing its actions}. This assumption poses severe limitations to the practical applicability of \expdom, which
also needs to solve a sequence of set cover
problems as a subroutine. 
% Therefore, later Exp3-SET and Exp3-DOM algorithms, with the regret bounds of order $\sqrt{\alpha T\ln
% \nodes}$, were proposed in \cite{alon2013from}. The drawback of these algorithms is that the regret bound of Exp3-SET
% algorithm holds only for an undirected observability graph and Exp3-DOM algorithm needs to know an observability graph
% before a player picks an action. Moreover the Exp3-DOM algorithm needs to know a
% dominating set of nodes for the graph, which is computationally expensive.
% There has been also some 
% work~\cite{caron2012leveraging,buccapatnam2014stochastic},
% that considers an (easier) stochastic counterpart of the original problem
% of~\cite{mannor2011from}.

In the present paper, we offer two computationally and
information-theoretically efficient algorithms for bandit
problems with directed observation systems. Both of our algorithms circumvent
the costly exploration phase required by \expdom by a trick that we will refer
to IX as in Implicit eXploration. Accordingly, we
name our algorithms \expfive and \fplix, which are variants of the well-known \exph \cite{auer2002bandit} and \fpl
\cite{KV05} algorithms enhanced with implicit exploration. Our first algorithm
\expfive is specifically
designed\footnote{\expfive can also be efficiently implemented for some 
specific combinatorial decision sets even with $m>1$,
see, e.g., \citet{CL12} for some examples.}  to work in the setting of
\citet{alon2013from} with $m=1$ and does not need
to solve any set cover problems or have any sort of prior knowledge concerning the observation systems chosen by the
adversary.\footnote{However, it is still necessary to have access to the 
observability graph to construct low bias estimates of losses, but only after 
the action is selected.}  \fplix, on the other hand, does need 
either to solve set cover problems or have a prior upper bound on the 
independence numbers of the observability graphs, but can be computed 
efficiently for a wide range of true combinatorial problems with $m>1$.
We note that our algorithms do not even 
need to know the number of rounds $T$ and our regret bounds
scale with the \emph{average} independence number $\bar{\alpha}$ of the graphs played by the adversary rather than the
largest of these numbers. They both employ adaptive learning rates
and unlike \expdom, they do not need to use a doubling trick to be anytime or
to aggregate outputs of multiple algorithms to optimally set their learning
rates. Both algorithms achieve regret guarantees of
$\tilde\OO(m^{3/2}\sqrt{\bar{\alpha} T })$
% $O(m^{3/2}\sqrt{\bar{\alpha} T \log T \log d})$
in the combinatorial setting, which becomes
$\tilde\OO(\sqrt{\bar{\alpha} T})$
% $O(\sqrt{\bar{\alpha} T \log T \log d})$
in the simple setting.

Before diving into the main content, we give an important graph-theoretic 
statement that we will rely on when
analyzing both of our algorithms. The lemma is a generalized version of Lemma~13
of \citet{alon2013from} and its proof
is given in Appendix~\ref{app:graphlemma}.
\begin{lemma}\label{lem:generalqbound}
Let $G$ be a directed graph with vertex set $V=\{1,\dots,\nodes\}$. Let 
$N_i^-$ be the
in-neighborhood of node $i$, i.e., the set of nodes $j$ such that $(j \ra i)\in
G$.
Let $\alpha$ be the independence number
of $G$ and $p_{1}$,\dots,$p_{\nodes}$ are numbers from $[0,1]$ such that $\sumi
p_i\leq m$. Then
$$\sumi\frac{p_i}{\frac{1}{m}p_i+\frac{1}{m}P_i+c}\leq2m\alpha\log\left(1+\frac{
m\lceil
\nodes^2/c\rceil+\nodes}{\alpha}\right)+2m,$$
where $P_i = \sum_{j\in N_i^-}p_j$ and $c$ is a positive constant.
\end{lemma}

%%%%%%%%%%%%%%%%%%%%%%%%%
%%%%%%%%%%%%%%%%%%%%%%%%%
\vspace{-0.5em}
\section{Multi-armed bandit problems with side information}
\vspace{-0.5em}
%%%%%%%%%%%%%%%%%%%%%%%%%
%%%%%%%%%%%%%%%%%%%%%%%%%

In this section, we start by the simplest setting fitting into our
framework, namely the multi-armed bandit problem with side observations.
We  provide intuition about the implicit exploration procedure behind our
algorithms and describe \expfive, the most natural algorithm based on
the IX trick. 
The problem we consider is defined as follows. In each round
$t=1,2,\dots,T$, the environment assigns
a loss vector $\bloss_t\in[0,1]^\nodes$ for $d$ actions and also 
selects an observation system described by the
directed graph $G_t$. Then, based on its previous observations (and likely some
external source of randomness) the
learner selects action $I_t$ and subsequently incurs and observes loss $\loss_{t,I_t}$. Furthermore, the learner also
observes the losses $\loss_{t,j}$ for all $j$ such that $(I_t \ra j)\in G_t$, 
denoted by the indicator $O_{t,i}$.
% The history of interaction between the learner and the environment up to time
% $t$ is recorded in the history
Let $\F_{t-1} = \sigma(I_{t-1},\dots,I_1)$ capture the interaction 
history up to time $t$.
As usual in online settings \cite{CBLu06:book}, the performance 
is measured in terms of (total expected) regret, which is the difference between
a total loss received and the total loss of the best single action chosen in
hindsight,
\[
% \vspace{-0.5em}
 R_T = \max_{i\in[d]} \EE{\sum_{t=1}^T \left(\loss_{t,I_t} - \loss_{t,i}\right)},
% \vspace{-0.5em}
\]
where the expectation integrates over the random choices made by the learning
algorithm. \citet{alon2013from} adapted the well-known \exph algorithm of
\citet{auer2002bandit} for this precise problem.
Their algorithm, \expdom, works by maintaining a weight $w_{t,i}$ for each
individual arm $i\in[d]$ in each round,
and selecting $I_t$ according to the distribution
\[
 \PPc{I_t = i}{\F_{t-1}} = (1-\gamma) p_{t,i} + \gamma \mu_{t,i} = (1-\gamma) \frac{w_{t,i}}{\sum_{j=1}^d w_{t,j}} +
\gamma \mu_{t,i},
\]
where $\gamma\in(0,1)$ is parameter of the algorithm and $\bmu_t$ is an \emph{exploration distribution} whose role we
will shortly clarify.
% After each round, \expdom computes the probabilities
% \[
%  o_{t,i} = \EEc{O_{t,i}}{\F_{t-1}} = \PPc{(I_t \ra i)\in G_t}{\F_{t-1}}
% \]
% and uses them to define the loss estimates
 After each round, \expdom defines the loss estimates
 \[
 \hloss_{t,i} = \frac{\loss_{t,i}}{o_{t,i}}\II{(I_t \ra i)\in G_t}
\quad \mbox{where} \quad o_{t,i} = \EEc{O_{t,i}}{\F_{t-1}} = \PPc{(I_t \ra 
i)\in G_t}{\F_{t-1}}
 \]
for each $i\in[d]$. 
These loss estimates are then used to update the weights for all $i$ as
%  \[
 $$w_{t+1,i} = w_{t,i} e^{-\gamma \hloss_{t,i}}.$$
%  \]
It is easy to see that these loss estimates $\hloss_{t,i}$ are unbiased
estimates of the true losses whenever $p_{t,i}>0$
holds for all $i$. This requirement along with another important technical issue
justify the presence of the exploration
distribution $\bmu_t$. The key idea behind \expdom is to compute a \emph{dominating set} $D_t\subseteq[d]$ of the
observability graph $G_t$ in each round, and define $\bmu_t$ as the uniform 
distribution over $D_t$. This choice ensures
that $o_{t,i} \ge p_{t,i} + \gamma/|D_t|$, a crucial requirement for the
analysis of~\cite{alon2013from}.
In what follows, we propose an exploration scheme that does not need any fancy
computations but, more importantly, works \textit{without any prior
knowledge of the observability graphs}.

% 
% In this section, we introduce an algorithm, based on Exp3, with a setting used in \cite{SM11,alon2013from}. Our task is
% to choose an action $I_t$ in each time step. The action consists of exactly one out of $\nodes$ components. Then we
% receive the loss of this action, and possibly observe the losses of some other components due to the observability
% graph. This feedback scheme is between the bandit and the full information scheme. While the analysis of the regret
% bound of our algorithm shows, that theoretical performance of our algorithm is comparable to the theoretical performance
% existing algorithms, the main contribution is that, unlike the other existing algorithms, our algorithm does not require % the access to the observability graph, to compute dominating set of the observability graph, using doubling trick
% or that the observability graph to be undirected. Therefore our algorithm is easier to implement and computationally
% less complex.

%%%%%%%%%%%%%%%%%%%%%%%%%
\vspace{-0.5em}
\subsection{Efficient learning by implicit exploration}
\vspace{-0.5em}
%%%%%%%%%%%%%%%%%%%%%%%%%

In this section, we propose the simplest exploration scheme imaginable, which consists of \emph{merely pretending to
explore}. Precisely, we simply sample our action $I_t$ from the distribution defined as
\begin{equation}\label{eq:Itsamp}
\PPc{I_t = i}{\F_{t-1}} = p_{t,i} = \frac{w_{t,i}}{\sum_{j=1}^d w_{t,j}},
\end{equation}
without explicitly mixing with any exploration distribution. Our key trick is to define the loss estimates for
all arms $i$ as
\[
 \hloss_{t,i} = \frac{\loss_{t,i}}{o_{t,i} + \gamma_t} \II{(I_t \ra i)\in G_t},
\]
where $\gamma_t>0$ is a parameter of our algorithm.
It is easy to check that $\hloss_{t,i}$ is a \emph{biased} estimate of $\loss_{t,i}$. The nature of this bias, however,
is very special. First, observe that $\hloss_{t,i}$ is an \emph{optimistic} estimate of $\loss_{t,i}$ in the sense that
$
\EEc{\hloss_{t,i}}{\F_{t-1}}\le \loss_{t,i}$. That is, our bias always ensures 
that,
 on expectation,
we underestimate the loss of any fixed arm
$i$. Even more importantly, our loss estimates also satisfy
\begin{equation}
\begin{split}\label{eq:lbias}
\EEcc{\sum_{i=1}^d p\ti\hloss\ti}{\F_{t-1}} & = \sum_{i=1}^d \pti\loss\ti +
\sum_{i=1}^d \pti\loss\ti\left(\frac{\oti}{\oti+\gammat}-1\right)
\\
&= \sum_{i=1}^d \pti\loss\ti - \gamma_t \sum_{i=1}^d \frac{\pti\loss\ti}{\oti+\gammat},
%&\geq \sum_{i\in V} p_{t,i}\loss_{t,i} - \sum_{i\in V}p_{t,i}\loss_{t,i}\left(\frac{\gamma}{q_{t,i}}\right)\\
\end{split}
\end{equation}
that is, the bias of the estimated losses \emph{suffered by our algorithm} is directly controlled by $\gamma_t$. As we
will see in the analysis, it is sufficient to control the bias of our own estimated performance as long as we can
guarantee that the loss estimates associated with any fixed arm are optimistic---which is precisely what we have.
Note that this slight modification ensures that the denominator of $\hloss_{t,i}$ is lower bounded by $p_{t,i} +
\gamma_t$, which is a very similar property as the one achieved by the exploration scheme used by \expdom.
We call the above loss estimation method \emph{implicit exploration} or IX, as it gives rise to the same effect as
explicit exploration without actually having to implement any exploration policy. In fact, explicit and implicit
explorations can both be regarded as two different approaches for 
bias-variance tradeoff: while explicit
exploration biases the \emph{sampling distribution} of $I_t$ to reduce the variance of the loss estimates, implicit
exploration achieves the same result by biasing \emph{the loss estimates themselves}.

From this point on, we take a somewhat more predictable course and define our algorithm \expfive as a variant of \exph
using the IX loss estimates. One of the twists is that \expfive is actually
based on the adaptive learning-rate variant of \exph proposed
by \citet{aucbge02}, which avoids the necessity of prior knowledge of the 
observability graphs in order to set a proper learning rate.
This algorithm is 
defined by setting $\hL_{t-1,i} =
\sum_{s=1}^{t-1} \hloss_{s,i}$ and for all $i\in[d]$ computing the weights as
 \[
 w_{t,i} = (1/d)e^{-\eta_t \hL_{t-1,i}}.
\]
 These weights are then used to construct the sampling distribution of $I_t$ as
defined in
\eqref{eq:Itsamp}. The resulting \expfive algorithm is shown as
Algorithm~\ref{alg:expix}.

 \begin{wrapfigure}{r}{0.5\textwidth}
  \vspace{-4.5em}
\begin{minipage}{0.5\textwidth}
   \begin{algorithm}[H]
\caption{\expfive}
\label{alg:expix}
\begin{algorithmic}[1]
 \STATE \textbf{Input:} Set of actions $\Sw = [d]$,
 \STATE \quad parameters
$\gammat\in(0,1)$, $\eta_t>0$ for $t\in[T]$.
% \STATE \textbf{Initialization:} $\hLoss_{0,i} \gets 0$ for  $i\in [d]$
\FOR{$t = 1$ {\bfseries to} $T$}
\STATE  $w\ti \gets (1/d)\exp{(-\etat\hL_{t-1,i})}$ for $i\in
[d]$
 \STATE An adversary privately chooses losses $\loss\ti$ for $i\in[d]$
and generates a graph $G_t$
\STATE $W_t \gets \sum_{i=1}^d w_{t,i}$
\STATE $p_{t,i} \gets w_{t,i}/W_t$
\STATE Choose $I_t  \sim \bp_t=
(p_{t,1},\dots,p_{t,\nodes})$
\STATE Observe graph $G_t$
\STATE Observe pairs $\{i,\loss_{t,i}\}$ for $(I_t \ra
i)\in G_t$
\STATE  $\oti \gets \sum_{(j \ra i)\in
G_t}\ptj$ for  $i\in [d]$
\STATE $\hloss\ti \gets
\frac{\loss_{t,i}}{\oti+\gammat}\I_{\{(I_t \ra i)\in G_t\}}$ for $i\in
[d]$
% \STATE  $\hLoss\ti \gets \hLoss_{t-1,i}+ \hloss\ti$ for  $i\in [d]$
\ENDFOR

\end{algorithmic}
\end{algorithm}
    \end{minipage}
 \vspace{-1em}
\end{wrapfigure}

% \begin{remark}
% We can set the parameters $\etat$ and $\gammat$ as
% $$\etat = \gammat = \sqrt{\frac{\ln \nodes}{\nodes+\sum_{\tau=1}^{t-1}Q'_\tau}},$$
% for all $t\leq T$. Note that $Q'_\tau$ for $\tau<t$ can be computed in time $t$. Where
% $$Q'_t = \sum_{i\in V}\frac{\pti}{\oti+\gammat}.$$
% \end{remark}

%%%%%%%%%%%%%%%%%%%%%%%%%
\vspace{-0.5em}
\subsection{Performance guarantees for \expfive}\label{sec:expix}
% \vspace{-0.5em}
%%%%%%%%%%%%%%%%%%%%%%%%%
Our analysis follows the footsteps of~\citet{auer2002bandit}
and~\citet{gyorfi07unbounded}, who provide an improved
analysis of the adaptive learning-rate rule proposed by~\citet{aucbge02}.
However, a technical subtlety will force us to proceed a little differently than these standard proofs: for achieving
the tightest possible bounds and the most
efficient algorithm, we need to tune our learning rates according to some random quantities that depend on the
performance of \expfive. In fact, the key quantities in our analysis are the terms
\[
 Q_t = \sum_{i=1}^d \frac{\pti}{\oti+\gammat},
\]
which depend on the interaction history $\F_{t-1}$ for all $t$.
Our theorem below gives the performance guarantee for \expfive using a 
parameter setting adaptive to the values of $Q_t$.
A full proof of the theorem is given in the
supplementary material.
\begin{theorem}\label{thm:expix}
Setting $
 \etat = \gammat = \sqrt{(\log{\nodes})/(\nodes+\sum_{s = 1}^{t-1}Q_s})$
 , the regret of \expfive
satisfies 
\begin{align}
\regret \leq 4\EE{\sqrt{\pa{d+\textstyle\sumt Q_t}\log d}}.
\label{eq:banditmaineq}
\end{align}
\end{theorem}
\begin{proof}[Proof sketch]
Following the proof of Lemma~1 in \citet{gyorfi07unbounded}, we can prove that
\begin{align}\label{eq:sumpt}
 \sum_{i=1}^d p_{t,i} \hloss_{t,i} \le \frac{\eta_t}{2} \sum_{i=1}^d p_{t,i} \pa{\hloss_{t,i}}^2 + \left(\frac{\log
W_t}{\etat}-\frac{\log W_{t+1}}{\eta_{t+1}}\right).
\end{align}
Taking \emph{conditional} expectations, using Equation~\eqref{eq:lbias} and
summing up both sides, we get
\[
 \sum_{t=1}^T \sum_{i=1}^d p_{t,i} \loss_{t,i} \le 
 \sumt \pa{\frac{\eta_t}{2} + \gamma_t} Q_t + \sumt \EEcc{\left(\frac{\log W_t}{\etat}-\frac{\log
W_{t+1}}{\eta_{t+1}}\right)}{\F_{t-1}}.
\]
Using Lemma~3.5 of \citet{aucbge02} and plugging in $\etat$ and $\gammat$, this 
becomes
\[
 \sum_{t=1}^T \sum_{i=1}^d p_{t,i} \loss_{t,i} \le 
 3\sqrt{\pa{d + \textstyle\sumt Q_t}\log d} + \sumt \EEcc{\left(\frac{\log W_t}{\etat}-\frac{\log
W_{t+1}}{\eta_{t+1}}\right)}{\F_{t-1}}.
\]
Taking expectations on both sides, the second term on the right hand side telescopes into
\[
\begin{split}
 \EE{\frac{\log W_{1}}{\eta_{1}}-\frac{\log W_{T+1}}{\eta_{T+1}}} \leq
\EE{-\frac{\log w_{T+1,j}}{\eta_{T+1}}} =
\EE{\frac{\log \nodes}{\eta_{T+1}}} + \EE{\hLoss_{T,j}}
\end{split}
\]
for any $j\in[d]$, giving the desired result as
\[
 \sum_{t=1}^T \sum_{i=1}^d p_{t,i} \loss_{t,i} \le \sum_{t=1}^T \loss_{t,j} +
 4\EE{\sqrt{\pa{d + \textstyle\sumt Q_t}\log d}},
\]
where we used the definition of $\eta_T$ and the optimistic property of the loss
estimates.
\end{proof}

% 
% \begin{theorem}\label{thm:mainExp}
% The regret of \expfive\ satisfies
% $$R(T)=\E\left[L_{A,T}-L_{k,T}\right]\leq \frac{\ln \nodes}{\eta_{T}} +  \sumt\E\left[\gammat Q'_t\right] +
% \frac{1}{2}\sumt\E\left[\etat Q'_t\right]. $$
% 
% Setting $\etat = \gammat = \sqrt{\ln{\nodes}/\left(\nodes+\sum_{\tau = 1}^{t-1}Q'_\tau\right)}$, we get
% 
% \begin{align}
% \regret \leq \E\left[L_{A,T}-L_{k,T}\right]\leq \frac{5}{2}\sqrt{\ln\nodes\sumt\E[Q'_t]}. \label{eq:banditmaineq}
% \end{align}
% \end{theorem}

%%%%%%%%%%%%%%%%%%%%%%%%%

Setting $m=1$ and $c=\gamma_t$ in Lemma \ref{lem:generalqbound}, gives
the following
\emph{deterministic} upper bound on each~$Q_t$.

%%%%%%%%%%%%%%%%%%%%%%%%%

% \begin{lemma}\label{lem:banditqbound}
% Let $G_t=(V,E_t)$ be a directed graph, with vertex set $V=\{1,\dots,\nodes\}$, and arc set $E_t$. Let $N\ti^-$ be the
% in-neighborhood of node $i$, i.e., the set of nodes $j$ such that $(j,i)\in E_t$. Let $\alpha_t$ be the independence
% number
% of $G_t$ and $p_{t,1}$,\dots,$p_{t,\nodes}$ be a probability distribution defined over $V$. Then
% $$Q'_t = \sumi\frac{\pti}{\gamma_t+\pti+P\ti}\leq2\alpha_t\ln\left(1+\frac{\lceil
% \nodes^2/\gammat\rceil+\nodes}{\alpha_t}\right)+2,$$
% where $P\ti = \sum_{j\in N\ti^-}\ptj$.
% \end{lemma}

\begin{lemma}\label{lem:banditqbound}
For all $t\in[T]$, 
$$Q_t = \sum_{i=1}^d \frac{\pti}{o_{t,i} +
\gamma_t}\leq2\alpha_t\log\left(1+\frac{\lceil
\nodes^2/\gammat\rceil+\nodes}{\alpha_t}\right)+2.$$
\end{lemma}

%%%%%%%%%%%%%%%%%%%%%%%%%

% \begin{proof}
% 
% \end{proof}

%%%%%%%%%%%%%%%%%%%%%%%%%

Combining Lemma~\ref{lem:banditqbound} with Theorem~\ref{thm:expix} we prove
our main result concerning the regret
of \expfive.

%%%%%%%%%%%%%%%%%%%%%%%%%

\begin{corollary}\label{cor:mainExp}
The regret  of \expfive satisfies
$$\regret \leq 4\sqrt{\pa{d + 2\textstyle\sumt\left(H_t \alpha_t+1\right)}\log
d},$$
where
\[
 H_t = \log\left(1+\frac{\lceil \nodes^2\sqrt{t\nodes/\log \nodes}\rceil+\nodes}{\alpha_t}\right) =\OO(\log (dT)).
\]
\end{corollary}

%%%%%%%%%%%%%%%%%%%%%%%%%
% 
% \begin{proof}
% Using Lemma \ref{lem:banditqbound} to upperbound \ref{eq:banditmaineq} we get
% $$\regret \leq \frac{5}{2}\sqrt{\ln{\nodes}\sumt\left[ 2\alpha_t\ln\left(1+\frac{\lceil \nodes^2/\gammat\rceil+\nodes}{\alpha_t}\right)+2\right]}.$$
% Moreover we can lowerbound $\gammat$ as
% 
% $$\gammat = \sqrt{\frac{\ln{\nodes}}{d+\sum_{\tau = 1}^{t-1}\E\left[Q'_t\right]}}\geq\sqrt{\frac{\ln{\nodes}}{\sum_{\tau = 1}^{t}d}} = \sqrt{\ln\nodes/t\nodes}.$$
% 
% We used
% $$Q'_t = \sumi\frac{\pti}{\oti+\gammat}\leq\sumi\frac{\oti}{\oti+\gammat}\leq\sumi1 = \nodes.$$
% 
% \end{proof}
%%%%%%%%%%%%%%%%%%%%%%%%%
\vspace{-0.5em}
\section{Combinatorial semi-bandit problems with side observations}
\vspace{-0.5em}
We now turn our attention to the setting of online combinatorial optimization (see
\cite{koolen10comphedge,CL12,audibert2014regret}). In this variant of the 
online learning problem, the learner has access
to a possibly huge action set $\Sw\subseteq\ev{0,1}^\nodes$ where each
action is represented by a binary vector $\bv$ of dimensionality $\nodes$. In 
what follows, we assume that $\onenorm{\bv}\le m$
holds for all $\bv\in\Sw$ and some $1\le m\ll d$, with the case $m=1$ 
corresponding to the multi-armed bandit setting considered in the
previous section. 
In each round $t=1,2,\dots,T$ of the decision process, the learner picks an action $\bV_t\in\Sw$ and incurs a loss of
$\bV_t\transpose\bloss_t$. 
% where $\bloss_t\in[0,1]^\nodes$ is a loss vector chosen by the environment. 
At the end of the
round,
the learner receives some feedback based on its decision $\bV_t$ and the loss vector $\bloss_t$. 
The regret of the learner
% in this setting
is defined as
\[
 R_T = \max_{\bv\in\Sw} \EE{\sum_{t=1}^T \pa{\bV_t - \bv}\transpose\bloss_t}.
\]
Previous work has considered the following feedback schemes in the combinatorial
setting:
\vspace{-0.5em}
\begin{itemize}
 \item The full information scheme where the learner gets to observe $\bloss_t$ regardless of
the chosen action. The minimax optimal regret of order $m\sqrt{T\log d}$ here is
achieved by \comphedge algorithm
of \citep{koolen10comphedge}, while the Follow-the-Perturbed-Leader (\fpl)
 \citep{KV05,Han57} was shown to enjoy a
regret of order $m^{3/2}\sqrt{T\log d}$ by \citep{neu2013efficient}.
 \item The semi-bandit scheme where the learner gets to observe the components
$\loss_{t,i}$ of the loss vector where $V_{t,i} = 1$, that is, the losses along the components chosen by the learner at
time $t$. As shown by \citep{audibert2014regret}, \comphedge achieves a
near-optimal $O(\sqrt{mdT\log d})$ regret
guarantee, while \citep{neu2013efficient} show that FPL enjoys a bound of
$O(m\sqrt{dT\log d})$.
 \item The bandit scheme where the learner only observes its own loss $\bV_t\transpose\bloss_t$. There are currently no
known efficient algorithms that get close to the minimax regret in this setting---the reader is referred to
\citet{audibert2014regret} for an overview of recent results.
\end{itemize}
\vspace{-0.5em}
In this section, we
% extend the framework of \citet{SM11} and \citet{alon2013from} to
define a new
feedback scheme
situated between the semi-bandit and the full-information schemes. In particular, we assume that the learner gets
to observe the losses of some other components not included in its own decision
vector $\bV_t$. Similarly to the model of~\citet{alon2013from}, the relation
between the chosen action and the side observations are given by a directed 
observability graph $G_t$ (see example in
Figure~1). We refer to this feedback scheme as \emph{semi-bandit
with side observations}. 
While our theoretical results stated in the previous section continue to hold in 
this setting, combinatorial \expfive could rarely be
implemented efficiently---we refer to \cite{CL12,koolen10comphedge} for some 
positive examples.
As one of the main concerns in this paper is computational efficiency, we take a different approach:
we propose a variant of \fpl that efficiently implements the idea of implicit exploration in
combinatorial semi-bandit problems with side observations.

\vspace{-0.5em}
\subsection{Implicit exploration by geometric resampling}
\vspace{-0.5em}
In each round $t$, \fpl~bases its decision on some estimate $\hbL_{t-1}=\sum_{s=1}^{t-1}\hbl_s$ of the total losses
$\bL_{t-1} = \sum_{s=1}^{t-1} \bloss_s$ as follows:
\begin{equation}\label{eq:fpl}
 \bV_t = \argmin_{\bv\in\Sw} \bv\transpose\left(\eta_t \hbL_{t-1} - \bZ_t\right).
\end{equation}
Here, $\eta_t>0$ is a parameter of the algorithm and $\bZ_t$ is a perturbation vector with components drawn
independently from an exponential distribution with unit expectation. The power
of \fpl~lies in that it only requires an oracle that solves the (offline)
optimization problem $\min_{\bv\in\Sw} \bv\transpose \bloss$ and thus can
be
used to turn any efficient offline solver into an online optimization algorithm with strong guarantees.
To define our algorithm precisely, we need some further notation.
  We redefine $\F_{t-1}$ to be $\sigma(\bV_{t-1},\dots,\bV_1)$, 
 $O_{t,i}$ to be the indicator of the observed \textit{component} and let
% to be the 
% interaction % history between the learner and the environment until the 
% beginning
% of round $t$. 
% We define $O_{t,i}$ as the indicator of the event that component $i$ is observed 
% in round $t$ and let
\[
 q_{t,i} = \EEc{V_{t,i}}{\F_{t-1}} \qquad\mbox{and}\qquad o_{t,i} = \EEc{O_{t,i}}{\F_{t-1}}.
\]
The most crucial point of our algorithm is the construction of our loss estimates. To implement the idea of implicit
exploration by optimistic biasing,  we apply a modified version of the geometric resampling method of
\citet{neu2013efficient} constructed as follows: Let $\bO'_t(1),\bO'_t(2),\dots$ 
be independent copies\footnote{Such
independent copies can be simply generated by sampling independent copies of $\bV_t$ using the \fpl rule \eqref{eq:fpl}
and then computing $\bO_t'(k)$ using the observability 
$G_t$. Notice that this procedure requires no interaction between the learner
and the environment, although each sample requires an oracle access.} of $\bO_t$
and let
$U_{t,i}$ be geometrically distributed random variables for all $i=[d]$ with
parameter $\gamma_t$. We let
\begin{equation}\label{eq:resamp}
 K_{t,i} = \min\left(\ev{k: O_{t,i}'(k) = 1}\cup\ev{U_{t,i}}\right)
\end{equation}
and define our loss-estimate vector $\hbl_t\in\real^d$ with its $i$-th element 
as
\begin{equation}\label{eq:grix}
 \hloss_{t,i} = K_{t,i} O_{t,i} \loss_{t,i}.
\end{equation}
By definition, we have $\EEc{K_{t,i}}{\F_{t-1}} = 1/(o_{t,i} +
(1-o_{t,i})\gamma_t)$, implying that our loss
estimates are \emph{optimistic} in the sense that they lower bound the losses in expectation:
\[
 \EEcc{\hloss_{t,i}}{\F_{t-1}} = \frac{o_{t,i}}{o_{t,i} + (1-o_{t,i})\gamma_t} \loss_{t,i} \le \loss_{t,i}.
\]
Here we used the fact that $O_{t,i}$ is independent of $K_{t,i}$ and has expectation $o_{t,i}$ given $\F_{t-1}$.
We call this algorithm Follow-the-Perturbed-Leader with Implicit
eXploration (\fplix, Algorithm~\ref{alg:fpl}).

Note that the geometric resampling procedure can be terminated as soon as $K_{t,i}$ becomes well-defined for all $i$
with $O_{t,i}=1$. As noted by
\citet{neu2013efficient}, this requires generating at most $d$ copies of $\bO_t$ 
on expectation. As each of these copies
requires one access to the linear optimization oracle over $\Sw$, we conclude
that the expected running time of \fplix
is at most $d$ times that of the expected running time of the oracle. A
high-probability guarantee of the
running time can be obtained by observing that
$U_{t,i}\le \log\pa{\frac1\delta}/\gamma_t$ holds with
probability at least $1-\delta$ and thus we can stop sampling after at most
$d\log\pa{\frac d\delta}/\gamma_t$
steps with probability at least $1-\delta$.

% \begin{algorithm}
% \caption{\fplix}
% \begin{algorithmic}[1]
% \STATE \textbf{Input: } Set of actions $\Sw$, parameters $\gammat\in(0,1)$ and
% $\eta_t>0$ for all $t\in[T]$.
% \STATE \textbf{Initialization:} $\hbl_1 \gets 0_N$
% \FOR{$t = 1$ {\bfseries to} $T$}
% \STATE An adversary privately chooses losses $\loss\ti$ for all $i\in[d]$ and
% generates a graph $G_t$
% \STATE Draw the components of $\bZ_t$ independently from a standard exponential distribution
% \STATE Play action $\bV_t \gets \argmin_{\bv\in\Sw} \bv\transpose
% \pa{\eta_t\hbL_{t-1}-\bZ_t}$
% \STATE Receive loss $\bV_t\transpose\bloss_t$
% \STATE Observe pairs $\{i,\loss_{t,i}\}$ for all $i$, such that $(j \ra i)\in
% G_t$ and $\bv(I_t)_j = 1$
% \STATE Compute $K_{t,i}$ for all $i\in [d]$ using Equation~\ref{eq:resamp}
% \STATE $\hloss\ti \gets  K\ti O\ti\loss\ti$
% \ENDFOR
% \end{algorithmic}
% \label{alg:fpl}
% \end{algorithm}

 \begin{wrapfigure}{r}{0.5\textwidth}
  \vspace{-3em}
\begin{minipage}{0.5\textwidth}
   \begin{algorithm}[H]
\caption{\fplix}
\label{alg:fpl}
\begin{algorithmic}[1]
\STATE \textbf{Input: } Set of actions $\Sw$,
\STATE \quad parameters $\gammat\in(0,1)$,
$\eta_t>0$ for  $t\in[T]$.
% \STATE \textbf{Initialization:} $\hbl_0 \gets 0_d$
\FOR{$t = 1$ {\bfseries to} $T$}
\STATE An adversary privately chooses losses $\loss\ti$ for all $i\in[d]$ and
generates a graph $G_t$
\STATE Draw $Z_{t,i} \sim \mbox{Exp}(1)$ for all $i\in[d]$
% independently from a standard exponential
\STATE $\bV_t \gets \argmin_{\bv\in\Sw} \bv\transpose
\pa{\eta_t\hbL_{t-1}-\bZ_t}$
\STATE Receive loss $\bV_t\transpose\bloss_t$
\STATE Observe graph $G_t$
\STATE Observe pairs $\{i,\loss_{t,i}\}$ for all $i$, such that $(j \ra i)\in
G_t$ and $\bv(I_t)_j = 1$
\STATE Compute $K_{t,i}$ for all $i\in [d]$ using Eq.~\eqref{eq:resamp}
\STATE $\hloss\ti \gets  K\ti O\ti\loss\ti$
\ENDFOR
\end{algorithmic}
\end{algorithm}
    \end{minipage}
   \vspace{-1em}
\end{wrapfigure}

\vspace{-0.5em}
\subsection{Performance guarantees for \fplix}\label{sec:fplix}
\vspace{-0.5em}
The analysis presented in this section combines some techniques used 
by~\citet{KV05,HuPo04}, and~\citet{neu2013efficient} for analyzing
\fpl-style learners. Our proofs also heavily rely on some specific properties of 
the IX loss estimate defined in
Equation~\ref{eq:grix}. The most important difference from the analysis presented in Section~\ref{sec:expix} is that
now we are not able to use random learning rates as we cannot compute the values corresponding to $Q_t$
efficiently. In fact, these values are observable in the information-theoretic sense, so we could prove
bounds similar to Theorem~\ref{thm:expix} had we had access to infinite computational resources. As our focus in this
paper
is on computationally efficient algorithms, we choose to pursue a different path. In particular, our learning rates
will be tuned according to efficiently computable approximations $\talpha_t$ of the respective independence numbers
$\alpha_t$ that satisfy $\alpha_t/C \le \talpha_t \le \alpha_t \le d$ for some $C\ge 1$. For the sake of simplicity,
we analyze the algorithm in the oblivious adversary model.
The following theorem states the performance guarantee for \fplix in terms of 
the learning rates and random variables of
the form
\[
 \tQ_t(c) = \sum_{i=1}^d \frac{q_{t,i}}{o_{t,i} + c}.
\]

\begin{theorem}\label{thm:fplix}
 Assume $\gamma_t\le 1/2$ for all $t$ and $\eta_1\ge\eta_2\ge\dots\ge
\eta_T$.
The regret of \fplix satisfies
 \[
  R_T \le \frac{m\left(\log d+1\right)}{\eta_T} + 4m \sum_{t=1}^T \eta_t \EE{\tQ_t\pa{\frac{\gamma_t}{1-\gamma_t}}} +
\sum_{t=1}^T \gamma_t \EE{\tQ_t(\gamma_t)}.
 \]
\end{theorem}

\begin{proof}[Proof sketch]
As usual for analyzing \fpl methods \citep{KV05,HuPo04,neu2013efficient}, we
first define a hypothetical
learner that uses a time-independent perturbation vector $\tbZ\sim\bZ_1$ and
has access to $\hbl_t$ on top of~$\hbL_{t-1}$
\[
 \tbV_t = \argmin_{\bv\in\Sw} \bv\transpose \pa{\eta_t \hbL_t - \tbZ}.
\]
Clearly, this learner is infeasible as it uses observations from the future. 
Also, observe that this learner does
not actually interact with the environment and depends on the predictions made 
by the actual learner only through the
loss estimates. 
By standard arguments, we can prove
\[
 \EE{\sum_{t=1}^T \left(\tV_t - \bv\right)\transpose \hbl_t}  \le \frac{m\left(\log d+1\right)}{\eta_T}.
\]
Using the techniques of \citet{neu2013efficient}, we can relate the performance 
of
$\bV_t$ to that of $\tbV_t$, which we can further upper bound after a long 
and tedious calculation as
\[
 \EEcc{(\bV_t - \tbV_t)\transpose \hbl_t}{\F_{t-1}} \le \eta_t\, 
\EEcc{\left(\tV_{t-1}\transpose
\hbl_{t}\right)^2}{\F_{t-1}} \le 
4m\eta_t\EEcc{\wt{Q}_t\pa{\frac{\gamma}{1-\gamma}}}{\F_{t-1}}.
\]
% \[
%  \EE{(\bV_t - \tbV_t)\transpose \hbl_t} \le 4m\eta_t\EE{\wt{Q}_t\pa{\frac{\gamma}{1-\gamma}}}.
% \]
The result follows by observing that $\EEcc{\bv\transpose\hbl_t}{\F_{t-1}} \le 
\bv\transpose\bloss_t$ for any fixed $\bv\in\Sw$ by
the optimistic property of the IX estimate and also from the fact that by  the 
definition of
the estimates we infer that
\[
 \EEcc{\tV_{t-1}\transpose \hbl_t}{\F_{t-1}} \ge 
\EEcc{\bV_t\transpose\bloss_t}{\F_{t-1}} - \gamma_t
\EE{\tQ_t(\gamma_t)}.
\vspace{-1.5em}
\]
\end{proof}
\vspace{-0.5em}
The next lemma shows a suitable upper bound for the last two terms in the  bound
of Theorem~\ref{thm:fplix}. It follows from observing that $o_{t,i}\ge 
(1/m)\sumtji\qtj$ and applying Lemma~\ref{lem:generalqbound}.

\begin{lemma}\label{lem:combqbound}
For all $t\in[T]$ and any $c\in(0,1)$,
$$\wt{Q}_t(c) = \sumi\frac{\qti}{\oti+c}\leq 2m\alpha_t\log\left(1 +
\frac{m\lceil\nodes^2/c\rceil+d}{\alpha_t}\right) + 2m.$$
\end{lemma}
% This statement is proved in the appendix. 
We are now ready to state the main result of this section, which
is obtained by combining Theorem~\ref{thm:fplix}, 
Lemma~\ref{lem:combqbound}, and Lemma~3.5 of
\citet{aucbge02} applied to the following upper bound
\[
\begin{split}
 \sum_{t=1}^T \frac{\alpha_t}{\sqrt{d + \sum_{s = 1}^{t-1}\talpha_s}}
 &\le \sum_{t=1}^T \frac{\alpha_t}{\sqrt{\sum_{s = 1}^{t}\alpha_s/C}}
%  \\
%  &
\le 2 \sqrt{C\textstyle\sum_{t = 1}^{T} \alpha_t} \le 2 \sqrt{d + C\textstyle\sum_{t = 1}^{T}
\alpha_t}.
\end{split}
\]
\begin{corollary}
Assume that for all $t\in[T]$, $\alpha_t/C \le \talpha_t \le \alpha_t \le d$ for 
some $C>1$, and 
assume $md>4$. Setting $\eta_ t =
\gamma_t = \sqrt{\pa{\log{\nodes} + 1}/\pa{m\pa{\nodes+\sum_{s 
=1}^{t-1}\talpha_s}}}$,
the regret of \fplix satisfies
 \[
  R_T \le H m^{3/2} \sqrt{\pa{d + C\textstyle\sum_{t = 1}^{T} \alpha_t}(\log d + 
1)}, \quad  \mbox{where $H$ is $\OO(\log (mdT))$}.
%   + G \sqrt{m\pa{d + C\textstyle\sum_{t = 1}^{T} \alpha_t}(\log d + 1)},
 \]
%  where $H$ is $O(\log (mdT))$.
%  \[
%   R_T \le 8m\sqrt{m(T\bar{\alpha}_T \log\pa{...} + 1) (\log d + 1)} + 2\sqrt{m(T\bar{\alpha}_T \log\pa{...} + 1) (\log d
% + 1)}.
%  \]
\end{corollary}
% \begin{proof}
% The proof of the statement follows
% \end{proof}
% \begin{lemma}
% Let $\qti$ be the probability that we play component $i$ in time $t$ and $oti$ be the probability of the loss
% observation of a component $i$ in time $t$.
% Clearly $\qti<\oti$. Moreover, the inequality $\sumi\qti\leq m$ holds. Then for $\gammat\in(0,1)$ we get
% $$\sumi\frac{(1-\gammat)\qti}{(1-\gammat)\oti+\gammat}\leq 2m\alpha_t\ln\left(1 +
% \frac{m\lceil\nodes^2/\cgammat\rceil+N}{\alpha_t}\right) + 2m$$
% for $\cgammat = (1-\gammat)/\gammat$.
% \end{lemma}

% %%%%%%%%%%%%%%%%%%%%%%%%%
% %%%%%%%%%%%%%%%%%%%%%%%%%
% \vspace{-0.5em}
 \paragraph{Conclusion}
% \vspace{-0.5em}
% %%%%%%%%%%%%%%%%%%%%%%%%%
% %%%%%%%%%%%%%%%%%%%%%%%%%
We presented an efficient algorithm for learning with side observations
based on implicit exploration. This technique gave rise to multitude of
improvements. Remarkably, our algorithms no longer need to know the
observation system before choosing the action unlike the
method of~\cite{alon2013from}. Moreover, we extended the partial
observability model of~\cite{mannor2011from,alon2013from} to accommodate 
problems with
large and structured action sets and also gave an efficient algorithm for this
setting.

% \vspace{-0.5em}
\paragraph{Acknowledgements}
\label{sec:Acknowledgements}
% \vspace{-0.5em}
The research presented in this paper was supported by French Ministry of
Higher Education and Research, by European Community's
Seventh Framework Programme (FP7/2007-2013) under grant agreement n$^{\rm
o}$270327 (CompLACS), and by FUI project Herm\` es.

% The contribution of this paper is introduction of two algorithms for solving
% non-combinatorial and combinatorial bandit problem with side information
% provided by the underlying observability graph and analyses of the performance
% of these algorithms. Moreover the observability graph does not have to be
% disclosed before an action is picked. This is the main reason why these
% algorithms are computationally not expensive. The other advantage of these
% algorithms is that their theoretical performance and therefore, the regret
% bounds are comparable to the best algorithms which are solving the similar
% problems .

\bibliography{library,ngbib,allbib,shortconfs}

\begin{thebibliography}{}

\bibitem[Alon et~al., 2013]{alon2013from}
Alon, N., Cesa-Bianchi, N., Gentile, C., and Mansour, Y. (2013).
\newblock {From Bandits to Experts: A Tale of Domination and Independence}.
\newblock In {\em Neural Information Processing Systems (NeurIPS)}.

\bibitem[Audibert et~al., 2014]{audibert2014regret}
Audibert, J.~Y., Bubeck, S., and Lugosi, G. (2014).
\newblock {Regret in Online Combinatorial Optimization}.
\newblock {\em Mathematics of Operations Research}, 39:31--45.

\bibitem[Auer et~al., 2002a]{auer2002bandit}
Auer, P., Cesa-Bianchi, N., Freund, Y., and Schapire, R.~E. (2002a).
\newblock The nonstochastic multiarmed bandit problem.
\newblock {\em SIAM J. Comput.}, 32(1):48--77.

\bibitem[Auer et~al., 2002b]{aucbge02}
Auer, P., Cesa-Bianchi, N., and Gentile, C. (2002b).
\newblock Adaptive and self-confident on-line learning algorithms.
\newblock {\em Journal of Computer and System Sciences}, 64:48--75.

\bibitem[Cesa-Bianchi et~al., 1997]{cfhhsw:expert97}
Cesa-Bianchi, N., Freund, Y., Haussler, D., Helmbold, D., Schapire, R., and
  Warmuth, M. (1997).
\newblock How to use expert advice.
\newblock {\em Journal of the ACM}, 44:427--485.

\bibitem[Cesa-Bianchi and Lugosi, 2006]{CBLu06:book}
Cesa-Bianchi, N. and Lugosi, G. (2006).
\newblock {\em Prediction, Learning, and Games}.
\newblock Cambridge University Press, New York, NY, USA.

\bibitem[Cesa-Bianchi and Lugosi, 2012]{CL12}
Cesa-Bianchi, N. and Lugosi, G. (2012).
\newblock Combinatorial bandits.
\newblock {\em Journal of Computer and System Sciences}, 78:1404--1422.

\bibitem[Chen et~al., 2013]{CWY13}
Chen, W., Wang, Y., and Yuan, Y. (2013).
\newblock {Combinatorial Multi-Armed Bandit: General Framework and
  Applications}.
\newblock In {\em International Conference on Machine Learning (ICML)}, pages
  151--159.

\bibitem[Gy\"{o}rfi and Ottucs\'{a}k, 2007]{gyorfi07unbounded}
Gy\"{o}rfi, L. and Ottucs\'{a}k, G. (2007).
\newblock {Sequential prediction of unbounded stationary time series}.
\newblock {\em IEEE Transactions on Information Theory}, 53(5):1866--1872.

\bibitem[Hannan, 1957]{Han57}
Hannan, J. (1957).
\newblock {Approximation to Bayes Risk in Repeated Play}.
\newblock {\em Contributions to the Theory of Games}, 3:97--139.

\bibitem[Hutter and Poland, 2004]{HuPo04}
Hutter, M. and Poland, J. (2004).
\newblock {Prediction with Expert Advice by Following the Perturbed Leader for
  General Weights}.
\newblock In {\em Algorithmic Learning Theory (ALT)}, pages 279--293.

\bibitem[Kalai and Vempala, 2005]{KV05}
Kalai, A. and Vempala, S. (2005).
\newblock {Efficient algorithms for online decision problems}.
\newblock {\em Journal of Computer and System Sciences}, 71:291--307.

\bibitem[Koolen et~al., 2010]{koolen10comphedge}
Koolen, W.~M., Warmuth, M.~K., and Kivinen, J. (2010).
\newblock {Hedging structured concepts}.
\newblock In {\em Conference on Learning Theory (COLT)}, pages 93--105.

\bibitem[Littlestone and Warmuth, 1994]{LW94}
Littlestone, N. and Warmuth, M. (1994).
\newblock The weighted majority algorithm.
\newblock {\em Information and Computation}, 108:212--261.

\bibitem[Mannor and Shamir, 2011]{mannor2011from}
Mannor, S. and Shamir, O. (2011).
\newblock {From Bandits to Experts: On the Value of Side-Observations}.
\newblock In {\em Neural Information Processing Systems (NeurIPS)}.

\bibitem[Neu and Bart\'{o}k, 2013]{neu2013efficient}
Neu, G. and Bart\'{o}k, G. (2013).
\newblock {An Efficient Algorithm for Learning with Semi-bandit Feedback}.
\newblock In {\em Algorithmic Learning Theory (ALT)}, volume 8139 of {\em Lecture Notes in Computer
  Science}, pages 234--248.

\bibitem[Vovk, 1990]{Vov90}
Vovk, V. (1990).
\newblock {Aggregating strategies}.
\newblock In {\em Computational Learning Theory (COLT)}, pages 371--386.

\end{thebibliography}
\bibliographystyle{apalike}

\appendix
%%%%%%%%%%%%%%%%%%%%%%%%%
%%%%%%%%%%%%%%%%%%%%%%%%%

\section{Proof of Lemma~\ref{lem:generalqbound}}\label{app:graphlemma}

The proof relies on the following two statements borrowed from
\citet{alon2013from}.
\begin{lemma}{\rm (cf.~Lemma 10 of 
\cite{alon2013from})}\label{lem:indBoundForGraph}
~Let $G$ be a directed graph, with $V = \{1,\dots,\nodes\}$. Let $d_i^-$ 
be
the indegree of the node $i$ and $\alpha =
\alpha(G)$ be the independence number of $G$. Then
$$\sum_{i=1}^\nodes\frac{1}{1+d_i^-}\leq2\alpha\log\left(1+\frac{\nodes}{\alpha}
\right).$$
\end{lemma}

%%%%%%%%%%%%%%%%%%%%%%%%%

\begin{lemma}{\rm (cf.~Lemma 12 of \cite{alon2013from})}\label{lem:techLemma}
~If $a,\,b\geq 0$ and $a+b\geq B>A>0$, then
$$\frac{a}{a+b-A}\leq\frac{a}{a+b}+\frac{A}{B-A}$$
\end{lemma}

%%%%%%%%%%%%%%%%%%%%%%%%%

\begin{proof}
$$\frac{a}{a+b-A}-\frac{a}{a+b} =
\frac{aA}{(a+b)(a+b-A)}\leq\frac{A}{a+b-A}\leq\frac{A}{B-A}$$
\end{proof}

We are now ready to prove Lemma~\ref{lem:generalqbound}. Our proof is obtained
as a generalization of the
proof of Lemma~13 by~\citet{alon2013from}.

Let $M = \lceil \nodes^2/c\rceil$ and $d_i^-$ be the indegree of  node $i$.
We begin by constructing a discretization of the values $p_i$ for all $i$ such
that the discretized version of
$p_i$ satisfies $\hpi = k/M$ for some integer $k$ and $\hpi-1/M<p_i\leq\hpi$. By
straightforward algebraic
manipulations and the fact that $x/(x+a)$ is increasing in $x$ for 
nonnegative $x$ and $a$, we obtain the bound
\begin{align*}
\hspace{1cm}&\hspace{-1cm}\sumi\frac{p_i}{\frac{1}{m}p_i+\frac{1}{m}P_i+c} =
m\sumi\frac{p_i}{p_i + \sumj p_j + mc} \\
&\leq m\sumi\frac{\hpi}{\hpi + \sumj\hp_j + mc-d_i^-/M}	\\
&\leq m\sumi\frac{\hpi}{\hpi + \sumj\hp_j + mc} +
m\sumi\frac{d_i^-/M}{mc-d_i^-/M}	\\
&\leq m\sumi\frac{M\hpi}{M\hpi + \sumj M\hp_j } + 2m,
\end{align*}
where the second to last inequality holds by Lemma \ref{lem:techLemma} with 
$a = \hpi$, $b=\sumj\hp_j$, $A = d_i^-/M$, and $B=mc$.
It remains to find a suitable upper bound for the first sum on the right hand
side.
To this end, we construct a graph $G'$ from our original graph $G$, where that 
we
replace each node $i$ of $G$ by a
clique $C_i$ with $Mp_i$ nodes. In this expanded graph, we connect all vertices
in clique $C_i$ with all vertices in
$C_j$ if and only if there is an edge from $i$ to $j$ in original graph $G$.
Note that our new graph $G'$ has the same independence number $\alpha$ as the
original
graph $G$. Also observe that the indegree $\hdk$ of a node $k$ in clique $C_i$
is equal to $Mp_i - 1 + \sumj Mp_j$.
Therefore, the remaining term can be rewritten as
\begin{align*}
\sumi\frac{M\hpi}{M\hpi + \sumj M\hp_j } = \sumi\sum_{k\in C_i}\frac{1}{1 +
\hdk}
\end{align*}
which in turn can be bounded using Lemma \ref{lem:indBoundForGraph} by
$$2\alpha \ln\left(1 + \frac{\sumi M\hpi}{\alpha}\right)\leq2\alpha\ln\left(1 +
\frac{mM+d}{\alpha}\right).$$
Using this bound we get
$$\sumi\frac{p_i}{\frac{1}{m}p_i+\frac{1}{m}P_i+c} \leq 2m\alpha\ln\left(1 +
\frac{mM+d}{\alpha}\right) + 2m$$
as advertised.

%%%%%%%%%%%%%%%%%%%%%%%%%

\section{Full proof of Theorem~\ref{thm:expix}}

\begin{proof}[Proof (Theorem \ref{thm:expix})]
We start by introducing some notation. Let
\[
\hL_{t-1,i} = \sum_{s = 1}^{t-1}\hloss_{s,i} \qquad\mbox{and}\qquad
W'_{t} = \frac 1d \sumi e^{-\eta_{t-1}\hL_{t-1,i}}.
\]
Following the proof of Lemma~1 of \citet{gyorfi07unbounded}, we track the evolution of $\log W_{t+1}'/W_t$ to control
the regret. We have
\begin{align*}
\frac{1}{\etat}\log\frac{W'_{t+1}}{W_t} &=\frac{1}{\etat}\log
\sumi\frac{\frac{1}{\nodes}e^{-\eta_{t}\hL_{t,i}}}{W_t}
=\frac{1}{\etat}\log\sumi\frac{w\ti e^{-\etat\hloss\ti}}{W_t}
			\\
&=\frac{1}{\etat}\log\sumi\pti e^{-\eta_t\hloss\ti}	\leq
\frac{1}{\etat}\log\sumi\pti\left(1-\etat\hloss\ti+\frac{1}{2}
(\etat\hloss\ti)^2\right)	\\
&=\frac{1}{\etat}\log\left(1-\etat\sumi\pti\hloss\ti
+\frac{\etat^2}{2}\sumi\pti(\hloss\ti)^2\right),
\end{align*}
where we used the inequality $\exp(-x) \le 1 - x + x^2/2$ that holds for 
$x\ge 0$. 
% This is warranted by our
% choice of $\eta_t = \gamma_t$ and that $\hloss_{t,i}\le \gamma_t^{-1}$ holds 
% by the construction of the estimates.
Using the inequality $\log(1-x)\leq-x$ that holds for all $x$, we get	%, and summing over time $t$
\begin{align*}
\sumi\pti \hloss\ti&\leq\left[\frac{\log
W_t}{\etat} -\frac{\log W'_{t+1}}{\etat}\right]+\sumi\frac{\etat}{2}\pti(\hloss\ti)^2	\\
&=\left[\left(\frac{\log
W_t}{\etat} 
-\frac{\log W_{t+1}}{\eta_{t+1}}\right)+\left(\frac{\log W_{t+1}}{\eta_{t+1}}
-\frac{\log 
W'_{t+1}}{\etat}\right)\right]+\sumi\frac{\etat}{2}\pti(\hloss\ti)^2.	
\end{align*}

The second term in brackets on the right hand side can be bounded as
$$W_{t+1} = \sumi\frac{1}{\nodes}e^{-\eta_{t+1}\hL_{t,i}} =
\sumi\frac{1}{\nodes}\left(e^{-\eta_{t}\hL_{t,i}}\right)^{\frac{\eta_{t+1}}{
\etat}}\leq\left(\sumi\frac{1}{\nodes}e^{
-\eta_{t}\hL_{t,i}}\right)^{\frac{\eta_{t+1}}{\etat}} =
(W'_{t+1})^{\frac{\eta_{t+1}}{\etat}},$$
where we applied Jensen's inequality to the concave function
$x^{\frac{\eta_{t+1}}{\etat}}$ for $x\in \R$. The function is
concave since $\eta_{t+1}\leq\etat$ by definition. Taking logarithms in the
above inequality, we get
$$\frac{\log W_{t+1}}{\eta_{t+1}}-\frac{\log W'_{t+1}}{\etat}\leq 0.$$

Using this inequality, we prove Equation~\eqref{eq:sumpt} as
\[
\sumi\pti\hloss\ti\leq\frac{\etat}{2}\sumi\pti\left(\hloss\ti\right)^2+\left(\frac{\log W_t}{\etat}-\frac{\log
W_{t+1}}{\eta_{t+1}}\right).
\]
Taking \emph{conditional} expectations and summing up both sides over the time,
we get
\[
\EEcc{\sumt\sumi\pti\hloss\ti}{\!\F_{t-1}\!}\!\leq\EEcc{\sumt\frac{\etat}{2}
\sumi\pti\left(\hloss\ti\right)^2}{\!\F_{t-1}\!}
\!+\!\sumt\EEcc{\frac{\log W_t}{\etat}\!-\!\frac{\log
W_{t+1}}{\eta_{t+1}}}{\!\F_{t-1}\!}\!\!.
\]

The first term in the above inequality is controlled as
\begin{align*}
\EEcc{\sumi p\ti\hloss\ti}{\F_{t-1}} & = \sumi \pti\loss\ti +
\sumi\pti\loss\ti\left(\frac{\oti}{\oti+\gammat}-1\right) \\
&= \sumi \pti\loss\ti -
 \sumi\pti\loss\ti\left(\frac{\gammat}{\oti+\gammat}\right)\\
%  &\geq \sum_{i\in V} p_{t,i}\loss_{t,i} - \sum_{i\in
% V}p_{t,i}\loss_{t,i}\left(\frac{\gamma}{q_{t,i}}\right)\\
&\geq\sumi\pti\loss\ti - \gammat Q_t,
\end{align*}
while the first one on the right hand side as
\begin{align*}
\EEcc{\sumi\pti(\hloss\ti)^2}{\F_{t-1}} &=
\sumi\pti\frac{\loss\ti^2}{(\oti+\gammat)^2}\oti\leq
\sumi\pti\frac{\loss\ti^2}{(\oti+\gammat)\oti}\oti\\
&\leq \sumi\pti\frac{1}{(\oti+\gammat)\oti}\oti = \sumi\frac{\pti}{\oti+\gammat}
= Q_t.
\end{align*}
Combining these bounds yields
$$
\sum_{t=1}^T \sum_{i=1}^d p_{t,i} \loss_{t,i} \le 
 \sumt \pa{\frac{\eta_t}{2} + \gamma_t} Q_t + \sumt \EEcc{\left(\frac{\log W_t}{\etat}-\frac{\log
W_{t+1}}{\eta_{t+1}}\right)}{\F_{t-1}}.
$$

To proceed, we substitute the parameter choice $\etat = \gammat =
\sqrt{(\log\nodes)/(\nodes+\sum_{s=1}^{t-1}Q_s)}$ 
and use a standard algebraic lemma  \cite[Lemma~3.5]{aucbge02} to get
\begin{align*}
\sum_{t=1}^T \sum_{i=1}^\nodes p_{t,i} \loss_{t,i} &\leq
 3\sqrt{\pa{\nodes + \textstyle\sumt Q_t} \log\nodes } + \sumt 
\EEcc{\left(\frac{\log W_t}{\etat}-\frac{\log
W_{t+1}}{\eta_{t+1}}\right)}{\F_{t-1}}.
\end{align*}
Taking expectation on both sides, the second term on the right hand side telescopes into
\begin{align*}
\EE{\sumt\left(\frac{\log W_t}{\etat}-\frac{\log W_{t+1}}{\eta_{t+1}}\right)} &=
\EE{\frac{\log W_{1}}{\eta_{1}}-\frac{\log
W_{T+1}}{\eta_{T+1}}}\leq\EE{-\frac{\log w_{T+1,j}}{\eta_{T+1}}}	\\
&=\EE{\frac{-1}{\eta_{T+1}}\log \left( 
\frac{1}{\nodes}e^{-\eta_{T+1}\hLoss_{T,j}}\right)} =
\EE{\frac{\log \nodes}{\eta_{T+1}}} + \EE{\hLoss_{T,j}},
\end{align*}
for any $j\in [d]$, where we used that $W_{T+1}\geq w_{T+1,j} $ and $W_1 = 1$
since $w_{1,i} = 1/\nodes$ by definition for
all $i\in [d]$. Substituting $\eta_{T+1}$, we get
\begin{align*}
\EE{\sum_{t=1}^T \sum_{i=1}^d p_{t,i} \loss_{t,i}} &\leq 
 3\EE{\sqrt{\log\nodes\left(\nodes+\textstyle\sumt Q_t\right)}} +
\EE{\sqrt{\log\nodes\left(\nodes+\textstyle\sumt Q_t\right)}}
+\EE{\hLoss_{T,j}},
\end{align*}
which together with the fact that our estimates $\hLoss_{T,j}$  are optimistic 
yields the theorem.

%%%%%
%%%%%

%
%
%
%
%
%Therefore, applying these inequalities, we obtain
%\begin{align}
%\sum_{t=1}^T\sum_{i\in V}p_{t,i}\hloss\ti\leq\hLoss_{T,i} + \frac{\log
%\nodes}{\eta_{T+1}}
%+\sumt\sumi\frac{\etat}{2}\pti(\hloss\ti)^2. \label{eq:lossBuond}
%\end{align}
%
%Moreover, we can use that our loss estimates are optimistic
%$$\E\left[\hloss\ti\right] =
%\sum_{j:j\overset{t}{\to}i}p_{j,t}\frac{\loss_{t,i}}{\oti+\gammat} =
%\oti\frac{\loss_{t,i}}{\oti+\gammat}\leq \loss_{t,i},$$
%
%
%
%\begin{remark}
%\cite{alon2013from} used $Q_t$ in the form:
%$$\sum_{i\in V}\frac{\pti}{\oti}$$
%\end{remark}
%
%Using above inequalities to upperbound equation \ref{eq:lossBuond}, we get
%$$\E\left[L_{A,T}-L_{k,T}\right]\leq \frac{\ln \nodes}{\eta_{T}} +
% \sumt\E\left[\gammat Q'_t\right] +
%\frac{1}{2}\sumt\E\left[\etat Q'_t\right].$$
%
%\begin{align*}
%\E\left[\sum_{t=1}^T\left(\sum_{i\in V}p_{t,i}\loss_{t,i} -\gammat
%Q_t\right)\right]&\leq\E\left[
%\sumt\loss_{t,k}+\frac{\log
%\nodes}{\eta_{T+1}}+\sum_{t=1}^T\frac{\etat}{2}Q_t\right]		\\
%\E\left[L_{A,T}-L_{k,T}\right]&\leq \frac{\log \nodes}{\eta_{T}} +
%\sumt\E\left[\gammat Q_t\right] +
%\frac{1}{2}\sumt\E\left[\etat Q_t\right].
%\end{align*}

%$$\E\left[L_{A,T}-L_{k,T}\right]\leq \frac{5}{2}\sqrt{\log\nodes\sumt\E[Q_t]}$$

%The proof in finished by setting $\eta = \gamma = \sqrt{2\log
% \nodes/\sumt\E[Q'_t]}$.

\end{proof}

%%%%%
%%%%%

\section{Full proof of Theorem~\ref{thm:fplix}}\label{app:fplix}

We begin with a statement that concerns the performance of the imaginary
learner that predicts $\tbV_t$ in round $t$.
\begin{lemma}\label{lem:cheat}
Assume $\eta_1\ge\eta_2\ge\dots\ge\eta_T$. For any sequence of loss estimates,
the expected regret of the hypothetical
learner against any fixed action $\bv\in\Sw$ satisfies
 \[
\EE{\sum_{t=1}^T \left(\tV_t - \bv\right)\transpose \hbl_t}  \le
\frac{m\left(\log d+1\right)}{\eta_T}.
 \]
\end{lemma}
\begin{proof}
For simplicity, define $\beta_t = 1/\eta_t$ for $t\ge 1$ and $\beta_0 = 0$. We
start by applying the
classical follow-the-leader/be-the-leader lemma (see, e.g.,
\citep[Lemma~3.1]{CBLu06:book}) to the loss sequence defined as
$(\hbl_1-\tbZ\beta_1,\hbl_2 - \tbZ (\beta_2 - \beta_1),\dots,\hbl_T - \tbZ
(\beta_T - \beta_{T-1}))$ to obtain
\[
 \sum_{t=1}^T \tbV_t\transpose\pa{\hbl_t - \tbZ\pa{\beta_t - \beta_{t-1}}} 
\le \tbV_T\transpose\pa{\hbL_T - \tbZ\beta_T}
\le \bv\transpose\pa{\hbL_T - \tbZ\beta_T}.
\]
After reordering and observing that $-\bv\transpose\tbZ\le 0$, we get
\[
\begin{split}
 \sum_{t=1}^T \pa{\tbV_t - \bv}\transpose \hbl_t &\le \sum_{t=1}^T (\beta_t -
\beta_{t-1})\tbV_t\transpose\tbZ
 \\
 &\le \bigl\|\tbV_t\bigr\|_1\bigl\|\tbZ\bigr\|_\infty \sum_{t=1}^T (\beta_t -
\beta_{t-1}) =
\bigl\|\tbV_t\bigr\|_1\bigl\|\tbZ\bigr\|_\infty \beta_T.
\end{split}
\]
The result follows from using our uniform upper bound on $\onenorm{\bv}$ for all
$\bv$ and the well-known bound
$\EE{\bigl\|\tbZ\bigr\|_\infty}\le \log d+1$.
\end{proof}
The following result can be extracted from the proof of Theorem~1 of
\citet{neu2013efficient}.
\begin{lemma}\label{lem:price}
For any sequence of nonnegative loss estimates,
\[
 \EEcc{(\tV_{t-1} - \tV_t)\transpose \hbl_t}{\F_{t}} \le \eta_t\,
\EEcc{\left(\tV_{t-1}\transpose
\hbl_{t}\right)^2}{\F_{t}}.
\]
\end{lemma}
Using these two lemmas, we can prove the following lemma that upper bounds the
total expected regret of \fplix~in
terms of the sum of the variables
\[
 \tQ_t(c) = \sum_{i=1}^d \frac{q_{t,i}}{o_{t,i} + c}.
\]
\begin{lemma}\label{lem:lossbound}
 Assume that $\gamma_t\le 1/2$ for all $t$.
Then,
 \[
  \sum_{t=1}^T \EEcc{\bV_t\transpose \bloss_t}{\F_{t-1}}\le \sum_{t=1}^T
\EEcc{\tbV_t\transpose\hbl_t}{\F_{t-1}}+ 4m
\sum_{t=1}^T \eta_t \EE{\tQ_t\pa{\frac{\gamma_t}{1-\gamma_t}}} + \sum_{t=1}^T
\gamma_t \EE{\tQ_t(\gamma_t)}.
 \]
\end{lemma}

\begin{proof}
First, note that Lemma~\ref{lem:price} implies
\[
 \EEcc{(\tV_{t-1} - \tV_t)\transpose \hbl_t}{\F_{t-1}} \le \eta_t\,
\EEcc{\left(\tV_{t-1}\transpose
\hbl_{t}\right)^2}{\F_{t-1}}
\]
by the tower rule of expectation.
We start by observing that
\[
\begin{split}
 & \EEcc{\tV_{t-1}\transpose \hbl_t}{\F_{t-1}} =
 \EEcc{\sumi \qti \hloss\ti}{\F_{t-1}} =
 \EEcc{\sumi q_{t,i} \frac{\loss_{t,i}}{o_{t,i} + (1-o_{t,i})\gamma_t}
O_{t,i}}{\F_{t-1}}
 \\
 &\qquad\geq \EEcc{\sumi q_{t,i} \frac{\loss_{t,i}}{o_{t,i} + 
(1-o_{t,i})\gamma_t}
\left(O_{t,i} + (1-o_{t,i})\gamma_t\right)
 - \gamma_t \sumi q_{t,i} \frac{1-o_{t,i}}{o_{t,i} +
(1-o_{t,i})\gamma_t}}{\F_{t-1}}
 \\
 &\qquad\ge \sumi q_{t,i} \loss_{t,i} - \gamma_t \EEcc{\sumi \frac{q_{t,i}
(1-o_{t,i})}{o_{t,i} +
(1-o_{t,i})\gamma_t}}{\F_{t-1}}
 \\
 &\qquad\ge \sumi q_{t,i} \loss_{t,i} - \gamma_t \EEcc{\sumi 
\frac{q_{t,i}}{o_{t,i}
+
\gamma_t}}{\F_{t-1}} = \EEcc{\bV_t\transpose\bloss_t}{\F_{t-1}} - \gamma_t
\tQ_t(\gamma_t) .
\end{split}
\]
% To simplify some notation, let us fix a time $t$ and define $\bV$ as an
% independent copy of $\bV_t$. Notice
% that $\tbV_{t-1}$ is also identically distributed as $\bV$ and thus
% \vskip 0.5em
To simplify some notation, let us fix a time $t$ and define $\bV=\tbV_{t-1}$. 
We deduce that
\[
\begin{split}
&\EEcc{\left(\tbV_{t-1}\transpose \hbl_{t}\right)^2}{\F_{t-1}}
\\
&\qquad=
\EEcc{\sum_{j=1}^d\sum_{k=1}^d
\left(V_{j}\hloss_{t,j}\right)\left(V_{k}\hloss_{t,k}\right)}{\F_{t-1}}
\\
&\qquad=
\EEcc{\sum_{j=1}^d\sum_{k=1}^d
\left(V_{j}K_{t,j}O_{t,j}\loss_{t,j}\right)\left(V_{k}K_{t,k}O_{t,k}\loss_{t,k}
\right)}{\F_{t-1}}\quad\qquad\mbox{
(def.~of $\hbl_t$)}
\\
&\qquad\le
\EEcc{\sum_{j=1}^d\sum_{k=1}^d
\frac{K_{t,j}^2+K_{t,k}^2}{2}\left(V_jO_{t,j}\loss_{t,j}\right)\left(V_kO_{t,k}
\loss_{t,k}\right)}
{\F_{t-1}}\quad\ \ \mbox{($2K_{t,j}K_{t,k}\le K_{t,j}^2 + K_{t,k}^2$)}
\\
&\qquad\le
\EEcc{\sum_{j=1}^d\sum_{k=1}^d
K_{t,j}^2\left(V_jO_{t,j}\loss_{t,j}\right)\left(V_kO_{t,k}\loss_{t,k}\right)}
{\F_{t-1}}\qquad\qquad\quad\mbox{(symmetry of $j$ and $k$)}
\\
&\qquad\le
2\EEcc{\sum_{j=1}^d\frac{1}{(o_{t,j} +
(1-o_{t,j})\gamma_t)^2}\left(V_jO_{t,j}\loss_{t,j}\right)\sum_{k=1}^d
V_{k}\loss_{t,k}}{\F_{t-1}}\mbox{(def.~of $K_{t,j}$ and $O_{t,k}\le
1$)}
\\
&\qquad\le
2m\EEcc{\sum_{j=1}^d\frac{V_j\loss_{t,j}}{o_{t,j} +
(1-o_{t,j})\gamma_t}}{\F_{t-1}}
\\
&\qquad\le
2m\sum_{j=1}^d \frac{q_{t,j}}{o_{t,j} + (1-o_{t,j})\gamma_t} =
\frac{2m}{1-\gamma_t}
\sum_{j=1}^d \frac{q_{t,j}}{o_{t,j} + \gamma_t/(1-\gamma_t)}
\\
&\qquad=
 \frac{2m}{1-\gamma_t} \tQ_t\pa{\frac{\gamma_t}{1-\gamma_t}} \le 4m
\tQ_t\pa{\frac{\gamma_t}{1-\gamma_t}},
\end{split}
\]
where we used our assumption on $\gamma_t$ in the last line. The first statement
follows from combining the above terms
with Lemma~\ref{lem:price} and using
$\EEcc{\bv\transpose\hbl_t}{\F_{t-1}}\le\bv\transpose\bloss_t$ by the optimistic
property of the loss estimates $\hbl_t$.
\end{proof}

% \vspace{-1.5em}
\section{Proof of Lemma~\ref{lem:combqbound}}
% \vspace{-0.5em}
We start with proving the lower bound on
\[
 o_{t,i}\ge \frac 1m \sumtji\qtj.
\]
We prove this by first proving $O_{t,i}\ge(1/m)\sumtji V_{t,j}$ as follows:
First, assume that $O_{t,j} = 0$, in which
case the bound trivially holds, since both sides evaluate to zero by definition
of $O_{t,i}$. Otherwise, we have
\[
 \frac 1m \sumtji V_{t,j} \le \frac 1m \sum_{j=1}^d V_{t,j} \le 1 = O_{t,i},
\]
where we used $\sum_{j\in V} V_{t,j}\le m$ in the last inequality. Taking
expectations gives the desired lower bound on
$o_{t,i}$. Then we get
$$\sumi\frac{\qti}{\oti+c}\leq\sumi\frac{\qti}{\qti+\sumtj\qtj+c}
.$$
The proof is completed using  Lemma~\ref{lem:generalqbound}.

%%%%%%%%%%%%%%%%%%%%%%%%%
%%%%%%%%%%%%%%%%%%%%%%%%%

\end{document}